\documentclass{article}
\usepackage{amsmath}
\usepackage{amsthm}
\usepackage{algorithm, algorithmic}

\usepackage{amssymb}
\usepackage{url}
\usepackage{enumerate}
\usepackage{sidecap}

\usepackage{stmaryrd}
\usepackage{textcomp}

\usepackage[utf8]{inputenc} 
\usepackage[T1]{fontenc}    
\usepackage{hyperref}       
\usepackage{amsfonts}       
\usepackage{nicefrac}       
\usepackage{microtype}      

\usepackage[usenames, dvipsnames]{color}
\usepackage{graphicx}

\usepackage{textcomp}

\newtheorem{theorem}{Theorem}[section]

\newtheorem{lemma}[theorem]{Lemma}

\newtheorem{definition}[theorem]{Definition}

\usepackage[margin=1.7in]{geometry}

\newcommand{\nc}[1]{\newcommand{#1}}
\nc{\bs}[1]{\boldsymbol{#1}}
\nc{\nnp}{NN}
\nc{\ms}{\mathcal{C}}
\nc{\df}{\delta}
\nc{\emp}[1]{\emph{#1}}
\nc{\rplus}{\mathbb{R}_+}
\nc{\arc}{x}
\nc{\be}{\begin{equation*}}
\nc{\ee}{\end{equation*}}
\nc{\ntr}{T}
\nc{\xt}[1]{x_{#1}}
\nc{\lvt}[1]{\bs{\ell}_{#1}}
\nc{\nma}{K}
\nc{\act}[1]{a_{#1}}
\nc{\lct}[2]{\ell_{#1,#2}}
\nc{\lab}{y}
\nc{\reg}[1]{R(#1)}
\nc{\expt}[1]{\mathbb{E}[#1]}
\nc{\apr}{c}
\nc{\ars}{\mathcal{S}}
\nc{\tpg}{\gamma}
\nc{\dtb}{\gamma}
\nc{\pkn}{\Psi}
\nc{\cns}{\mathcal{X}}
\nc{\mrg}{\mathcal{M}}
\nc{\pra}{\rho}
\nc{\tmo}{\tilde{\mathcal{O}}}
\nc{\asp}{\Lambda}
\nc{\nat}{\mathbb{N}}
\nc{\tre}{\mathcal{Z}}
\nc{\arv}{v}
\nc{\vco}[1]{c(#1)}
\nc{\vde}[1]{d(#1)}
\nc{\rot}{r}
\nc{\arl}{u}
\nc{\lea}{\mathcal{L}}
\nc{\evv}[1]{\bs{\lambda}(#1)}
\nc{\evc}[2]{\lambda_{#1}(#2)}
\nc{\bone}{\bs{1}}
\nc{\la}{\leftarrow}
\nc{\lt}[1]{\hat{x}_{#1}}
\nc{\ebt}[1]{u_{#1}}
\nc{\mxl}[1]{\delta_{#1}}
\nc{\arbl}{h}
\nc{\nvd}[1]{h_{#1}}
\nc{\rne}[1]{\epsilon_{#1}}
\nc{\niv}[1]{z_{#1}}
\nc{\cht}[1]{q_{#1}}
\nc{\cit}[1]{f_{#1}}
\nc{\nct}[1]{g_{#1}}
\nc{\indi}[1]{\llbracket#1\rrbracket}
\nc{\noc}[1]{\xi(#1)}
\nc{\mm}{\bs{\Delta}}
\nc{\mmc}[2]{\Delta_{#1,#2}}
\nc{\seq}[2]{\langle#1\,|\,#2\rangle}
\nc{\set}[2]{\{#1\,|\,#2\}}
\nc{\crv}[1]{\sigma(#1)}
\nc{\vst}[1]{\mathcal{X}_{#1}}
\nc{\prt}[1]{{\uparrow}(#1)}
\nc{\tmc}[1]{\theta_{#1}}
\nc{\tm}[3]{\tau_{#1,#2}(#3)}
\nc{\ac}{a}
\nc{\pbt}[2]{\tilde{p}_{#1,#2}}
\nc{\bpbt}[1]{\tilde{\bs{p}}_{#1}}
\nc{\pbp}[2]{p_{#1,#2}}
\nc{\arf}{f}
\nc{\lr}{\eta}
\nc{\NN}{NN}
\nc{\nnf}[1]{n(#1)}
\nc{\pov}{\bs{y}}
\nc{\poc}[1]{y_{#1}}
\nc{\doc}[1]{\gamma_{#1}}
\nc{\cov}{\mathcal{C}}
\nc{\const}{g}
\nc{\consu}{g'}
\nc{\pcm}[1]{\Psi(#1)}
\nc{\conv}{g''}
\nc{\con}{f}
\nc{\pnt}[1]{p_{#1}}
\nc{\cst}{\mathcal{Z}}
\nc{\dpt}[1]{d_{#1}}
\nc{\ntm}[2]{\tau(#1,#2)}
\nc{\bal}[3]{\sigma_{#1,#2}(#3)}
\nc{\glo}{\lambda}
\nc{\hal}[1]{\lambda_{#1}}
\nc{\tp}[1]{\tilde{p}_{#1}}
\nc{\btp}{\tilde{\bs{p}}}
\nc{\bp}{\bs{p}}
\nc{\pc}[1]{p_{#1}}
\nc{\fv}{\bs{b}}
\nc{\fc}[1]{b_{#1}}
\nc{\hac}[1]{\hat{a}_{#1}}
\nc{\qd}[1]{q_{#1}}
\nc{\cstt}[1]{\mathcal{Z}_{#1}}
\nc{\mdf}{\delta}
\nc{\mga}{\tilde{\gamma}}
\nc{\lev}{\delta}
\nc{\lep}{d}
\nc{\nna}{c}
\nc{\nn}[1]{s_{#1}}
\nc{\mal}{h}
\nc{\les}[1]{\mathcal{H}_{#1}}
\nc{\aes}{\mathcal{H}}
\nc{\dlt}[1]{\delta_{#1}}
\nc{\hpv}{\hat{\bs{y}}}
\nc{\hpc}[1]{\hat{y}_{#1}}
\nc{\lef}{\mathcal{L}}
\nc{\des}[1]{\mathcal{D}_{#1}}
\nc{\anc}[1]{\mathcal{A}_{#1}}
\nc{\alg}{\textsc{HNN}}
\nc{\cbn}{\textsc{CBNN}}
\nc{\cts}{\mathcal{U}}
\nc{\ctb}{\mathcal{V}}
\nc{\ctc}{\mathcal{W}}
\nc{\zco}{z}
\nc{\zct}{\beta}
\nc{\ball}[1]{\zct\con^{#1}}
\nc{\ded}[2]{\mathcal{E}_{#1,#2}}
\nc{\dec}[1]{\mathcal{D}_{#1}}
\nc{\ctd}[1]{\mathcal{V}_{#1}}
\nc{\cte}[1]{\mathcal{W}_{#1}}
\nc{\ala}{\alpha}
\nc{\dst}{k}
\nc{\lepp}{\delta}
\nc{\phc}{\Phi}
\nc{\van}{\mathcal{E}}
\nc{\ndo}[1]{\gamma_{#1}}
\nc{\nmg}{\Omega}
\nc{\nco}{\Psi}
\nc{\asr}{\Lambda}
\nc{\xll}[1]{\mu_{#1}}
\nc{\nen}[1]{\nu_{#1}}
\nc{\otc}[1]{\mathcal{J}_{#1}}
\nc{\pla}[1]{y'_{#1}}
\nc{\bra}{\lambda}
\nc{\rad}[1]{\mathcal{Y}_{#1}}
\nc{\ypr}{y'}
\nc{\str}{\phi}
\nc{\thd}[1]{\theta_{#1}}
\nc{\lpd}[2]{\tilde{\ell}_{#1,#2}}
\nc{\mlm}[1]{\mu_{#1}}
\nc{\SLIV}{Slivkins2009ContextualBW}
\nc{\awi}{\eta}
\nc{\ctn}{\ctc}
\nc{\sea}[1]{\mathcal{Q}_{#1}}
\nc{\awn}{2}
\nc{\lgf}{\log_\con}
\nc{\bo}{\mathcal{O}}
\nc{\nan}{\textsc{NN}}
\nc{\fbal}{\mathcal{B}}
\nc{\sbal}{\mathcal{B}'}
\nc{\sbr}{r}
\nc{\bnr}{\epsilon}
\nc{\exd}{D}
\nc{\bll}[2]{\mathcal{B}(#1,#2)}

\title{A Hierarchical Nearest Neighbour Approach to Contextual Bandits}

\author{
Stephen Pasteris \\
The Alan Turing Institute\\
London UK\\
  \texttt{spasteris@turing.ac.uk} 
\and
Chris Hicks\\
  The Alan Turing Institute\\
London UK\\
  \texttt{c.hicks@turing.ac.uk}  
\and
Vasilios Mavroudis\\
  The Alan Turing Institute\\
London UK\\
  \texttt{vmavroudis@turing.ac.uk}
}

\begin{document}

\maketitle

\begin{abstract}
In this paper we consider the adversarial contextual bandit problem in metric spaces. The paper ``Nearest neighbour with bandit feedback" tackled this problem but when there are many contexts near the decision boundary of the comparator policy it suffers from a high regret. In this paper we eradicate this problem, designing an algorithm in which we can hold out any set of contexts when computing our regret term. Our algorithm builds on that of ``Nearest neighbour with bandit feedback" and hence inherits its extreme computational efficiency.
\end{abstract}

\section{Introduction}

We consider the contextual bandit problem in metric spaces. In this problem we have some (potentially unknown) metric space of bounded diameter. We assume that we have access to an oracle for computing distances. On each trial $t$ we are given a context $\xt{t}$ and must choose an action $\act{t}$ before observing the loss/reward generated by that action. In this paper the contexts are considered implicit and we define $\mmc{s}{t}$ to be the distance between $\xt{s}$ and $\xt{t}$.

This problem has been well-studied in the stochastic case (see e.g. \cite{\SLIV}, \cite{Reeve2018TheKN}, \cite{Perchet2011TheMB} and references therein). In this paper we consider the fully adversarial problem in which no assumptions are made at all about the metric space, context sequence, or loss sequence. As far as we are aware the first non-trivial result for the fully adversarial problem was given by the recent paper \cite{\NN} which bounds the regret with respect to any policy. This regret bound is fantastic when the contexts partition into well separated clusters and the policy is constant on each cluster. However, the bound is poor when there exist many contexts lying close to the decision boundary of the policy. In order to (partially) rectify this \cite{\NN} proposed using binning as a preprocessing step. We note that the optimal bin radius can be implicitly learnt via a type of doubling trick, although this was not discussed in \cite{\NN}. The problem with this is that for different parts of the metric space the optimal bin radii will be different. But \cite{\NN} can only learn a constant bin radius - leading to poor performance. In this paper we fully rectify this problem, designing a new (but related) algorithm \alg\ in which, in the loss bound, we can hold out any set of contexts (which we call a \emp{margin} - i.e. points near the decision boundary) when computing the regret term. In Section \ref{innsec} we give an example of how we improve over \cite{\NN}. We note, however, that in cases where the contexts partition into well separated clusters and the policy is constant on each cluster it may be advantageous to use simple nearest neighbour as in \cite{\NN}.

To achieve this improvement we will utilise the meta-algorithm \textsc{CBNN} of \cite{\NN} which receives, on each trial $t>1$, only some $\pnt{t}\in[t-1]$ (we note that in \cite{\NN} the notation $\nnf{t}$ was used instead). \cite{\NN} analysed the case for when $\pnt{t}$ is chosen such that $\xt{\pnt{t}}$ is an approximate nearest neighbour of $\xt{t}$ in the set $\set{\xt{s}}{s\in[t-1]}$. In this paper we use a different choice of $\pnt{t}$ which we call an approximate \emp{hierarchical nearest neighbour}. In approximate hierarchical nearest neighbour we will construct, online, a partition of of the trials seen so far into different \emp{levels}. On each trial $t$ the algorithm then uses approximate nearest neighbour on each level. $\pnt{t}$ will then be chosen, in a specific way, from one of these approximate nearest neighbours. We note that since our algorithm \alg\ is based on \textsc{CBNN} it inherits its extreme computational efficiency - having a per trial time complexity polylogarithmic in both the number of trials and number of actions when our dataset has an aspect ratio polynomial in the number of trials (which can be enforced by binning) and our metric space has bounded doubling dimension.

We note that, when using exact nearest neighbour, the process of inserting a given trial into our data-structure is essentially the same as the process of constructing a new bin in \cite{\SLIV}. However, the objectives of the data-structures are very different and they are analysed in very different ways (our analysis being far more involved than that of \cite{\SLIV}). Nevertheless, we cite \cite{\SLIV} as an inspiration for this paper.

We also note that the \cbn\ algorithm, upon which \alg\ is based, was inspired by the papers \cite{Delcher1995LogarithmicTimeUA}, \cite{Auer2002TheNM}, \cite{Matsuzaki2008MATHEMATICALET}, \cite{Pearl1982ReverendBO}, \cite{Freund1997UsingAC} and \cite{Herbster2021AGO}.

\section{The Issue with Nearest Neighbour}\label{innsec}

We now give an example of the issue with the algorithm (using binning with nearest neighbour) proposed by \cite{\NN}. We will call this algorithm (when implicitly learning the optimal bin radius via a form of doubling trick) \nan. For simplicity let's assume that the parameter $\pra$, of \nan\ and \alg, is set equal to a constant, although the argument easily extends to tuned parameter values as well. Consider two disjoint balls $\fbal,\sbal$ in $\mathbb{R}^2$ with radii equal to $1$ and $\sbr<1$ respectively. Assume that each ball has $\ntr$ contexts distributed uniformly over it. Suppose we have two actions and a comparator policy such that the decision boundary (of the policy) on each ball is a straight line going through the ball's centre. This is depicted in Figure 1.

\begin{figure}[t]\label{twoballs}
\includegraphics[scale=0.4]{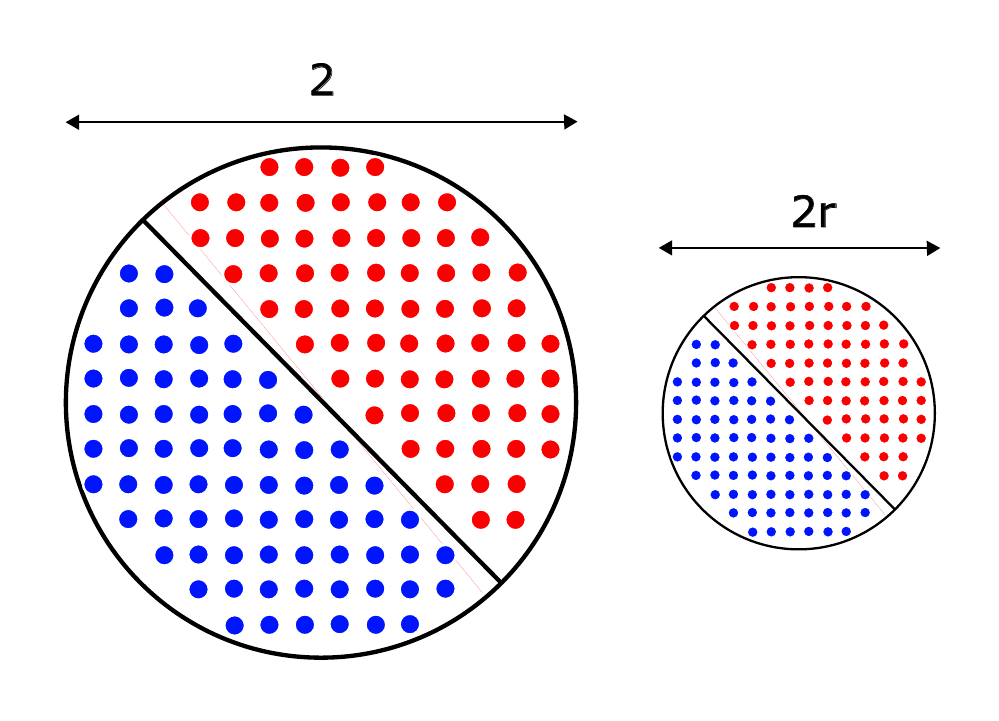}
\centering
\caption{An example of when \nan\ performs poorly. The colour of a context represents the action assigned to it by the comparator policy.}
\end{figure}

First let's analyse \nan\ when working on each ball as a seperate problem. Given a binning radius of $\bnr$, the regret of \nan\ on $\fbal$ is $\tilde{\mathcal{O}}(\sqrt{T}/\bnr+\bnr T)$ and the regret of \nan\ on $\sbal$ is $\tilde{\mathcal{O}}(\sbr\sqrt{T}/\bnr+\bnr T/\sbr)$. For both these problems the optimal value of $\bnr$ leads to a regret of $\tilde{\mathcal{O}}(\ntr^{3/4})$.

However, things change when \nan\ is working on both balls at the same time. In this case, given a binning radius of $\bnr$, the regret of \nan\ is:
\be
\tilde{\mathcal{O}}\left(\left(\sqrt{T}/\bnr+\bnr T\right)+\left(\sbr\sqrt{T}/\bnr+\bnr T/\sbr\right)\right)=\tilde{\mathcal{O}}\left(\sqrt{T}/\bnr+\bnr T/\sbr\right)
\ee
i.e. the sum of the regrets on both balls. This means that the optimal value of $\bnr$ leads to a regret of $\tilde{\mathcal{O}}(\ntr^{3/4}\sbr^{-1/2})$ which can be dramatically higher than if the balls were learnt separately. In fact, when $r\leq\ntr^{-1/2}$ this bound is vacuous - meaning that binning is not helping at all. Our algorithm \alg, however, has a regret of $\tilde{\mathcal{O}}(\ntr^{3/4})$ whenever $1/\sbr$ is polynomial in $\ntr$ - the same as if the balls were learnt seperately.

When working in $\mathbb{R}^\exd$ for $\exd>2$\,, the improvement of \alg\ over \nan\ is even stronger. To see this note that the regret of \nan\ when working on both balls is:
\be
\tilde{\mathcal{O}}\left(\left(\sqrt{T}/\bnr^{\exd-1}+\bnr T\right)+\left(\sqrt{T}(\sbr/\bnr)^{\exd-1}+\bnr T/\sbr\right)\right)=\tilde{\mathcal{O}}\left(\sqrt{T}/\bnr^{\exd-1}+\bnr T/\sbr\right)
\ee
so that the optimal value of $\bnr$ gives us a regret of $\tilde{\mathcal{O}}\left(\sbr^{-(\exd-1)/\exd}\ntr^{(2\exd-1)/(2\exd)}\right)$
which becomes vacuous at $\sbr\leq\ntr^{-1/(2\exd-2)}$. \alg, on the other hand, achieves a regret of $\tilde{\mathcal{O}}(\ntr^{(2\exd-1)/(2\exd)})$ whenever $1/\sbr$ is polynomial in $\ntr$  - the same as if the balls were learnt seperately.

\section{Problem Description}
We consider the following game between \emp{Nature} and \emp{Learner}. We have $\ntr$ \emp{trials} and $\nma$ \emp{actions}. Nature first chooses a matrix $\mm\in[0,1]^{\ntr\times\ntr}$ satisfying the following conditions:
\begin{itemize}
\item For all $s,t\in[\ntr]$ we have $\mmc{s}{t}=\mmc{t}{s}$.
\item For all $t\in[\ntr]$ we have $\mmc{t}{t}=0$.
\item For all $s,t\in[\ntr]$ we have $\mmc{s}{t}\geq0$. For simplicity we will assume, without loss of generality, that for all $s,t\in[\ntr]$ with $s\neq t$ we have $\mmc{s}{t}>0$. This is without loss of generality since if $\mmc{s}{t}=0$ then the trials $s$ and $t$ are equivalent. Dealing with equivalent trials is straightforward in the $\textsc{CBNN}$ algorithm of \cite{\NN} upon which our algorithm \alg\ is based. Trials equivalent to any proceeding trials will be ignored in the data-structure that we construct so have no effect on the computational complexity.
\item For all $r,s,t\in[\ntr]$ we have $\mmc{r}{t}\leq\mmc{r}{s}+\mmc{s}{t}$. This property is called the \emp{triangle inequality}.
\end{itemize}
We note that $\mm$ is a \emp{metric} over $[\ntr]$. Intuitively, every trial $t\in[\ntr]$ is implicitly associated with a \emp{context} $\xt{t}$ and for all $s,t\in[\ntr]$ we have that $\mmc{s}{t}$ is a measure of how similar $\xt{s}$ is to $\xt{t}$ (a smaller value of $\mmc{s}{t}$ means a greater similarity). For all trials $t\in[\ntr]$ and actions $\ac\in[\nma]$ Nature chooses a probability distribution $\lpd{t}{\ac}$ over $[0,1]$ and a \emp{loss} $\lct{t}{\ac}$ is then drawn from $\lpd{t}{\ac}$. We note that Learner has knowledge of only $\ntr$ and $\nma$ (although the requirement of knowledge of $\ntr$ can be removed by a simple doubling trick). The game then proceeds in $\ntr$ trials. On trial $t$ the following happens:
\begin{enumerate}
\item For all $s\in[t]$ Nature reveals $\mmc{s}{t}$ to Learner.
\item Learner chooses an action $\act{t}\in[\nma]$.
\item Nature reveals the loss $\lct{t}{\act{t}}$ to Learner.
\end{enumerate}
The aim of Learner is to minimise the cumulative loss:
\be
\sum_{t\in[\ntr]}\lct{t}{\act{t}}
\ee
Note that, since Nature has complete control over each distribution $\lpd{t}{\ac}$\,, our problem generalises the fully adversarial problem (which is the special case in which each distribution $\lpd{t}{\ac}$ is a delta function). We are considering this generalised problem in order for our bound to be better when there is an element of stochasticity in Nature's choices.

\section{The Algorithm}
We now describe our algorithm \alg. The algorithm takes parameters $\nna\geq1$ and $\pra>0$. We define $\con:=1/2$. Given a non-empty set $\aes\subseteq[\ntr]$ and a trial $t\in[\ntr]$, a $\nna$-\emp{nearest neighbour} of $t$ in the set $\aes$ is any trial $s\in\aes$ in which:
\be
\mmc{s}{t}\leq\nna\min\set{\mmc{r}{t}}{r\in\aes}
\ee
We utilise the algorithm \cbn\ \cite{\NN} with parameter $\pra$ as a subroutine. During the algorithm we will associate each trial $t\in[\ntr]$ with a number $\dpt{t}\in\nat\cup\{0\}$ and initialise by setting $\dpt{1}\la0$.
On each trial $t>1$ we do the following:

\begin{enumerate}
\item $\mal\la\max\set{\dpt{s}}{s\in[t-1]}$
\item For all $\lep\in[\mal]$ set $\les{\lep}\la\set{s\in[t-1]}{\dpt{s}=\lep}$
\item For all $\lep\in[\mal]$ let $\nn{\lep}$ be a $\nna$-nearest neighbour of $t$ in $\les{\lep}$
\item
Let $\lev$ be the maximum value of $\lep\in[\mal]$ such that $\mmc{\nn{\lep}}{t}\leq\con^{\lep}$
\item $\dpt{t}\la\lev+1$
\item $\pnt{t}\la\nn{\lev}$
\item Input $\pnt{t}$ into \cbn
\item Select $\act{t}$ equal to the output of \cbn
\item Receive $\lct{t}{\act{t}}$
\item Update \cbn\ with $\lct{t}{\act{t}}$
\end{enumerate}

We note that when $\nna>0$ we can maintain, for each $\lep\in[\mal]$\,, a \emp{navigating net} \cite{Krauthgamer2004NavigatingNS} over the set $\les{\lep}$ in order to rapidly find $\nn{\lep}$.

\section{Performance}
We now bound the expected cumulative loss of \alg. We first define the various constants used in this section. Let $\str$ be any value in $(0,1)$. We note that the algorithm has no knowledge of $\str$. We define:
\be
\con:=1/2~~~~~~;~~~~~~~\zct:=2/\str
\ee
\be
\bra:=(\str\zct+(1+\zct)/\con)/(\zct(1-\str))~~~~~~;~~~~~~\zco:=(1-\con)\con/(\awn\nna(1+\zct))
\ee
 We define the \emp{aspect ratio} of our dataset as:
\be
\asr:=\min\set{\mmc{s}{t}}{s,t\in[\ntr]\,\wedge\, s\neq t}
\ee
A \emp{policy} is any vector $\pov\in[\nma]^\ntr$. A \emp{margin} is any subset $\mrg\subseteq[\ntr]$. Suppose we have a policy $\pov\in[\nma]^\ntr$ and a margin $\mrg\subseteq[\ntr]$ such that there exists $s,t\in[\ntr]\setminus\mrg$ with $\poc{s}\neq\poc{t}$.
Note that the algorithm has no knowledge of either $\pov$ or $\mrg$. For all $t\in[\ntr]$ we define:
\be
\ndo{t}:=\min\set{\mmc{s}{t}}{s\in[\ntr]\setminus\mrg\,\wedge\,\poc{s}\neq\poc{t}}
\ee
We define:
\be
\nmg:=\min\set{\ndo{t}}{t\in[\ntr]}
\ee
and define $\nco$ as the maximum cardinality of any set $\ars\subseteq[\ntr]$ in which for all $s,t\in\ars$ with $s\neq t$ we have:
\be
\mmc{s}{t}> \zco\min(\ndo{s},\ndo{t})
\ee
For all $t\in[\ntr]$ define:
\be
\thd{t}:=\min\set{\mmc{s}{t}}{s\in[\ntr]\setminus\mrg}
\ee
and define:
\be
\rad{t}:=\set{\poc{s}}{s\in[\ntr]\,\wedge\,\mmc{s}{t}\leq\bra\thd{t}}~~~~~;~~~~~\xll{t}:=\max\set{\expt{\lct{t}{a}}}{a\in\rad{t}}
\ee
noting that for all $t\in[\ntr]\setminus\mrg$ we have $\thd{t}=0$ so that $\xll{t}=\expt{\lct{t}{\poc{t}}}$. \alg\ achieves the following performance:
\begin{theorem}\label{mainth}
The expected cumulative loss of \alg\ is bounded by:
\be
\sum_{t\in[\ntr]}\expt{\lct{t}{\act{t}}}\leq\sum_{t\in[\ntr]}\xll{t}+\tilde{\mathcal{O}}\left(\nco\ln(1/\nmg)^2+\left(\pra-\frac{\nco\ln(1/\nmg)}{\pra}\right)\sqrt{\nma\ntr}\right)
\ee
The running time of \alg\ is in $\tilde{\mathcal{O}}(\ntr)$ and, when $\nna>1$ and $\mm$ has bounded doubling dimension, is also in $\mathcal{O}(\ln(\ntr/\asr)^2\ln(\nma))$. \alg\ requires only $\mathcal{O}(\ntr\ln(\nma))$ space.
\end{theorem}

We now point out the effect of the choice of the margin $\mrg$ on our bound. Note that increasing $\mrg$ increases $\ndo{t}$ and $\thd{t}$ (for some trials $t\in[\ntr]$). The increase in $\ndo{t}$ (for some $t$) can decrease the value $\nco$ which helps us. However, the increase in $\thd{t}$ (for some $t$) can cause $\rad{t}$ to grow - potentially increasing $\xll{t}$ which hurts us. Hence, there will be a sweet-spot - the optimal margin $\mrg$. Trials in the optimal margin will correspond to contexts that are close to the decision boundary of the policy (but just \emp{how} close will depend on the location of the context and the density of contexts in its vicinity).

We note that we can use binning to enforce that $1/\asr$ is polynomial in $\ntr$ - hence ensuring polylogarithmic time per trial when $\mm$ has bounded doubling dimension and $\nna>1$. To do this we choose some $\epsilon$ with $1/\epsilon$ polynomial in $\ntr$ and, on any trial $t\in[\ntr]$ such that there exists $s\in[t-1]$ with $\mmc{s}{t}<\epsilon$ we treat $t$ as equivalent to $s$ (in \cbn) and ignore $t$ in our data-structure. We note, however, that this process can have an effect on the loss bound.

\nc{\dms}{d}
\nc{\reld}{\mathbb{R}^\dms}
\nc{\aco}{x}
\nc{\ra}{r}
\nc{\eun}[1]{\|#1\|}
\nc{\ext}{\tilde{y}}
\nc{\exm}{\tilde{\mrg}}
\nc{\db}[1]{\mathcal{D}(#1)}
\nc{\sep}{\epsilon}
\nc{\cnst}{C}
\nc{\nmb}{N}
\nc{\zi}[1]{v_{#1}}
\nc{\ri}[1]{r_{#1}}
\nc{\zxp}{x'}
\nc{\soc}{\mathcal{X}}
\nc{\bbr}{\xi}
\nc{\dbn}{\mathcal{D}}
\nc{\ran}[1]{\mathcal{J}(#1)}
\nc{\ire}[1]{R(#1)}

\section{When in Euclidean Space}
In order to give insight into Theorem \ref{mainth} we now analyse it in the case that the (implicit) contexts lie in the euclidean space $\reld$ (for some constant $\dms\in\nat$) and our metric is the euclidean metric, giving a relatively simple loss bound. We note, however, that we do not use the full power of Theorem \ref{mainth} here - for instance, we crudely bound $\xll{t}$ by $1$ when $t\in\mrg$.  

We make the following definitions. For all $\aco\in\reld$ let $\eun{\aco}$ be the euclidean norm of $\aco$. Given $\aco\in\reld$ and $\ra>0$ we define the \emp{ball}: 
\be
\bll{\aco}{\ra}:=\set{\zxp\in\reld}{\eun{\aco-\zxp}\leq\ra}
\ee
Here we assume that there exists a sequence of \emp{contexts} $\seq{\xt{t}}{t\in[\ntr]}\subseteq\bll{0}{1/2}$ such that for all $s,t\in[\ntr]$ we have $\mmc{s}{t}=\eun{\xt{s}-\xt{t}}$. An \emp{extended policy} is any function $\ext:\bll{0}{1/2}\rightarrow[\nma]$. Consider any such extended policy $\ext$. We define the \emp{decision boundary} as:
\be
\dbn:=\set{\aco\in\bll{0}{1/2}}{\forall\sep>0\,,\,\exists\, \zxp\in\bll{\aco}{\sep}:\ext(\zxp)\neq\ext(\aco)}
\ee
Theorem \ref{mainth} then gives us the following.
\begin{theorem}\label{maincor}
Choose any constants $\cnst>0$ and $\bbr>1$. Suppose we have some $\nmb\in\nat$ and any sequence $\seq{(\zi{i},\ri{i})}{i\in[\nmb]}\subseteq\reld\times\mathbb{R}_+$ with:
\begin{itemize}
\item $\dbn\subseteq\bigcup_{i\in[\nmb]}\bll{\zi{i}}{\ri{i}}$
\item $\min\set{\ri{i}}{i\in[\nmb]}\geq\ntr^{-\cnst}$
\end{itemize}
We define $\mrg$ to be the set of all $t\in[\ntr]$ in which there exists $i\in[\nmb]$ with $\xt{t}\in\bll{\zi{i}}{\bbr\ri{i}}$. The expected cumulative loss of \alg\ (with any constant $\nna$) is then bounded by:
\be
\sum_{t\in[\ntr]}\expt{\lct{t}{\act{t}}}\leq\sum_{t\in[\ntr]\setminus\mrg}\expt{\lct{t}{\ext(\xt{t})}}+|\mrg|+\tilde{\mathcal{O}}\left(\left(\pra-\frac{\nmb}{\pra}\right)\sqrt{\nma\ntr}\right)
\ee
\end{theorem}

In Figure 2 we give an example of the objects appearing in Theorem \ref{maincor}.

\begin{figure}[t]\label{eucfig}
\includegraphics[scale=0.4]{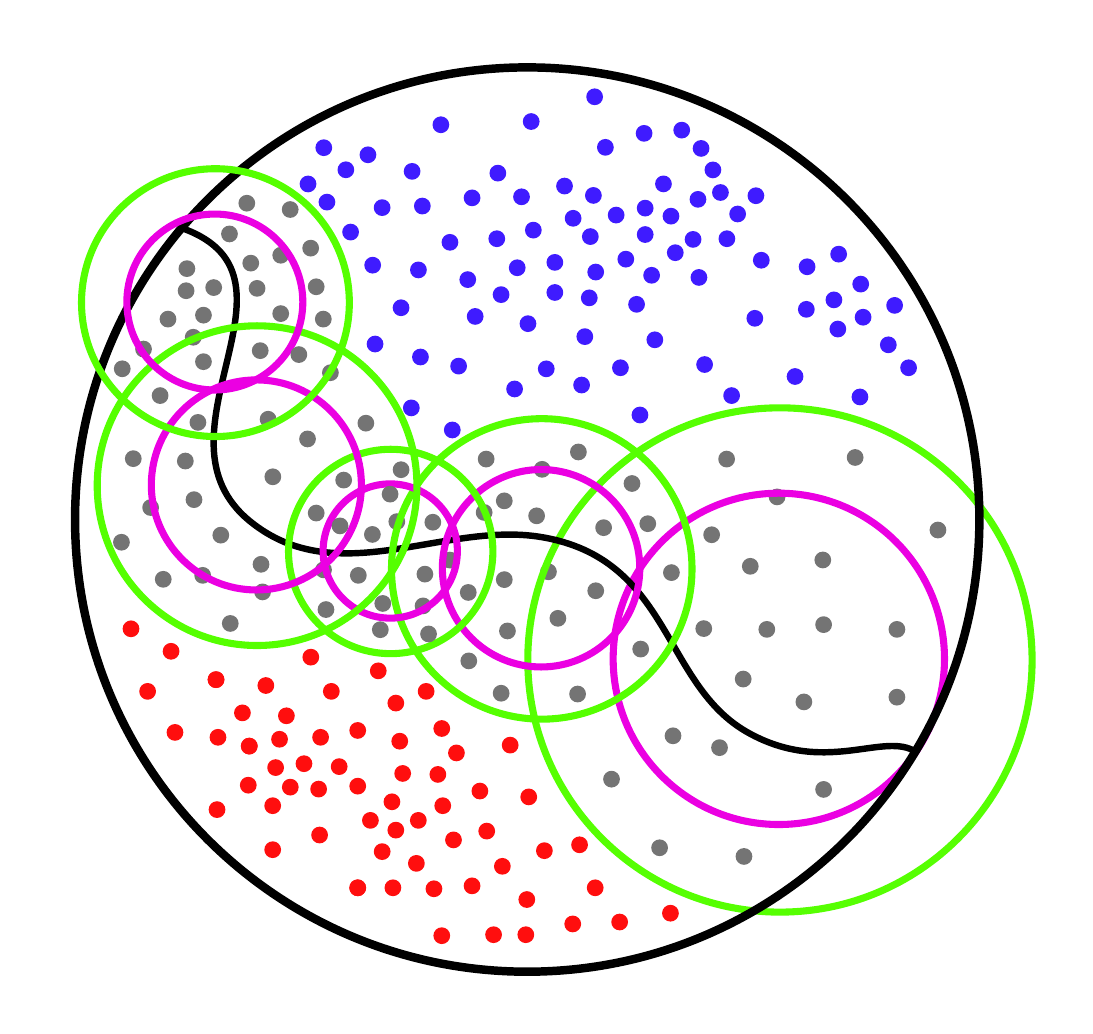}
\centering
\caption{An example with $\nma=2$ and $\dms=2$. Here we have $\nmb=5$. The black curve is the decision boundary $\dbn$. The purple balls are those in $\set{\bll{\zi{i}}{\ri{i}}}{i\in[\nmb]}$ and the green balls are those in $\set{\bll{\zi{i}}{\bbr\ri{i}}}{i\in[\nmb]}$. The grey contexts are those in $\set{\xt{t}}{t\in\mrg}$, the red contexts are those in $\set{\xt{t}}{t\in[\ntr]\setminus\mrg\,\wedge\,\ext(\xt{t})=1}$, and the blue contexts are those in $\set{\xt{t}}{t\in[\ntr]\setminus\mrg\,\wedge\,\ext(\xt{t})=2}$. Note that the purple balls cover the decision boundary and the grey contexts are those covered by the green balls.}
\end{figure}

\section{Proof of Theorem \ref{mainth}}

We now prove Theorem \ref{mainth}. We will often use the fact that, for all $t\in[\ntr]$\,, we have $\dpt{\pnt{t}}=\dpt{t}-1$ and $\mmc{t}{\pnt{t}}\leq\con^{\dpt{t}-1}$ in this analysis.

\begin{definition}
Consider the rooted tree with vertex set $[\ntr]$ such that, for all $t\in[\ntr]\setminus\{1\}$, we have that $\pnt{t}$ is the parent of $t$. Let $\lef$ be the set of leaves of this tree. Given $t\in[\ntr]$ we then define $\des{t}$ to be the set of all descendants of $t$ and define $\anc{t}$ to be the set of all ancestors of $t$.
\end{definition}

\begin{lemma}\label{lem1}
For all $r,t\in[\ntr]$ with $r\neq t$ and $\dpt{r}=\dpt{t}$ we have that $\mmc{r}{t}>\con^{\dpt{r}}/\nna$.
\end{lemma}

\begin{proof}
Suppose, for contradiction, the converse: that $\mmc{r}{t}\leq\con^{\dpt{r}}/\nna$. Without loss of generality assume $r<t$. Let $\mal:=\max\set{\dpt{s}}{s\in[t-1]}$ and for all $\lep\in[\mal]$ let $\nn{\lep}$ be as created by the algorithm on trial $t$. Let $q:=\nn{\dpt{r}}$. Since $q$ is a $\nna$-nearest neighbour of $t$ in the set $\set{s\in[t-1]}{\dpt{s}=\dpt{r}}$ (which contains $r$) we must have that
$\mmc{q}{t}\leq\nna\mmc{r}{t}\leq\con^{\dpt{r}}$.
But from the algorithm we have that $\dpt{t}-1$ is the maximum value of $\lep\in[\mal]$ such that $\mmc{\nn{\lep}}{t}\leq\con^{\lep}$ so since $\dpt{t}-1=\dpt{r}-1<\dpt{r}$ we have a contradiction.
\end{proof}

\begin{definition}
Define $\cts$ to be the set of all trials $t\in[\ntr]$ in which for all $r,s\in[\ntr]\setminus\mrg$ with $\mmc{r}{t}\leq\ball{\dpt{t}}$ and $\mmc{s}{t}\leq\ball{\dpt{t}}$ we have $\poc{r}=\poc{s}$.
\end{definition}

\begin{lemma}\label{cnanclem2}
Given $s,t\in[\ntr]$ with $s\in\cts$ and $t\in\des{s}$ we have $t\in\cts$.
\end{lemma}

\begin{proof}
Noting that $\dpt{t}\geq\dpt{s}$ we fix $s$ and prove by induction on $\dpt{t}$. When $\dpt{t}=\dpt{s}$ we have $t=s$ so the result is immediate. Now suppose, for some $\lep\geq\dpt{s}$\,, that the inductive hypothesis holds for all $t$ with $\dpt{t}=\lep$. Now take $t$ with $\dpt{t}=\lep+1$. Since $t\in\des{s}$ and $t\neq s$ we have $\pnt{t}\in\des{s}$. So since $\dpt{\pnt{t}}=\lep$ we have, by the inductive hypothesis, that $\pnt{t}\in\cts$. Now take any $q,r\in[\ntr]\setminus\mrg$ with $\mmc{q}{t}\leq\ball{\dpt{t}}$ and $\mmc{r}{t}\leq\ball{\dpt{t}}$. From the algorithm we have that $\mmc{t}{\pnt{t}}\leq\con^{\dpt{t}-1}$ and hence, by the triangle inequality, we have:
\be
\mmc{r}{\pnt{t}}\leq\mmc{r}{t}+\mmc{t}{\pnt{t}}\leq\zct\con^{\dpt{t}}+\con^{\dpt{t}-1}=(\zct\con+1)\con^{\dpt{t}-1}\leq\zct\con^{\dpt{t}-1}
\ee
Similarly we have $\mmc{q}{\pnt{t}}\leq\zct\con^{\dpt{t}-1}$.
So since $\dpt{\pnt{t}}=\dpt{t}-1$ and $\pnt{t}\in\cts$ we must have that $\poc{q}=\poc{r}$. Hence, we must have that $t\in\cts$ so the result holds by induction.
\end{proof}

\begin{definition}
Let $\ctb$ be the set of all $t\in[\ntr]$ such that either: 
\begin{itemize}
\item $t\in\cts$ and $\pnt{t}\notin\cts$
\item $t\in\lef$ and $t\notin\cts$
\end{itemize}
\end{definition}

\begin{definition}
Let $\ctn$ be the set of all $t\in\ctb$ such that there does not exist $s\in\ctb$ with $\dpt{s}>\dpt{t}$ and $\mmc{s}{t}\leq(\con^{\dpt{t}}-\con^{\dpt{s}})/(\awn\nna)$ 
\end{definition}

\begin{definition}
For any $t\in\ctn$ let $\sea{t}$ be equal to the set of all $s\in\ctb\setminus\ctn$ in which $\mmc{s}{t}\leq\con^{\dpt{s}}/(\awn\nna)$
\end{definition}

\begin{lemma}\label{fixlem1}
Given $s\in\ctb\setminus\ctn$ and $\lep\in\nat$ we either have that there exists $t\in\ctn$ with $s\in\sea{t}$ or that there exists some $r\in\ctb\setminus\ctn$ with $\dpt{r}\geq\lep$ and $\mmc{r}{s}\leq(\con^{\dpt{s}}-\con^{\dpt{r}})/(\awn\nna)$
\end{lemma}

\begin{proof}
If there exists $t\in\ctn$ with $s\in\sea{t}$ then we're done so assume otherwise. We prove by induction on $\lep$. We immediately have the result for $\lep=0$ by choosing $r:=s$. Now suppose, for some $\lep'\in\nat\cup\{0\}$\,,  that the inductive hypothesis holds when $\lep=\lep'$ and consider the case that $\lep=\lep'+1$. By the inductive hypothesis choose $q\in\ctb\setminus\ctn$ with $\dpt{q}\geq\lep'$ and $\mmc{q}{s}\leq(\con^{\dpt{s}}-\con^{\dpt{q}})/(\awn\nna)$. Since $q\in\ctb\setminus\ctn$ we have, by definition of $\ctn$, that there exists $u\in\ctb$ with $\dpt{u}>\dpt{q}$ and $\mmc{u}{q}\leq(\con^{\dpt{q}}-\con^{\dpt{u}})/(\awn\nna)$. By the triangle inequality we then have:
\be
\mmc{u}{s}\leq\mmc{u}{q}+\mmc{q}{s}\leq(\con^{\dpt{q}}-\con^{\dpt{u}})/(\awn\nna)+(\con^{\dpt{s}}-\con^{\dpt{q}})/(\awn\nna)=(\con^{\dpt{s}}-\con^{\dpt{u}})/(\awn\nna)
\ee
If it was the case that $u\in\ctn$ we would have, from this inequality, that $s\in\sea{u}$ which is a contradiction. Hence, we have $u\in\ctb\setminus\ctn$. Since $\dpt{u}>\dpt{q}$ and $\dpt{q}\geq\lep'$ we have $\dpt{u}\geq\lep'+1=\lep$. By the above inequality we then have the result by choosing $r:=u$. This completes the inductive proof.
\end{proof}

\begin{lemma}\label{fixlem2}
Given $s\in\ctb\setminus\ctn$ there exists $t\in\ctn$ with $s\in\sea{t}$.
\end{lemma}

\begin{proof}
Suppose, for contradiction, the converse. By Lemma \ref{fixlem1} we then have, for all $\lep\in\nat$, that there exists some $r\in\ctb\setminus\ctn$ with $\dpt{r}\geq\lep$. By choosing $\lep:=\ntr+1$ we then have that there exists $r\in[\ntr]$ with $\dpt{r}\geq\ntr+1$ which is impossible.
\end{proof}

\begin{lemma}\label{sizqlem}
For all $t\in\ctn$ and $s\in\sea{t}$ we have $\dpt{s}\leq\dpt{t}$.
\end{lemma}

\begin{proof}
Suppose, for contradiction, that there exists $t\in\ctn$ and $s\in\sea{t}$ with $\dpt{s}>\dpt{t}$. Then by definition of $\sea{t}$ we have $s\in\ctb$ and $\mmc{s}{t}\leq\con^{\dpt{s}}/(\awn\nna)$. Since $\dpt{t}\leq\dpt{s}-1$ we then have:
\be
(\con^{\dpt{t}}-\con^{\dpt{s}})/(\awn\nna)\geq(\con^{\dpt{s}-1}-\con^{\dpt{s}})/(\awn\nna)=(1/\con-1)\con^{\dpt{s}}/(\awn\nna)\geq(1/\con-1)\mmc{s}{t}
\ee
So since $\con=1/2$ we have $\mmc{s}{t}\leq(\con^{\dpt{t}}-\con^{\dpt{s}})/(\awn\nna)$ which, since $\dpt{s}>\dpt{t}$ and $s\in\ctb$, contradicts the fact that $t\in\ctn$.
\end{proof}

\begin{lemma}\label{fixlem3}
For all $t\in\ctn$ we have $|\sea{t}|\leq\dpt{t}+1$.
\end{lemma}

\begin{proof}
By Lemma \ref{sizqlem} all we need to prove is that if $q,r\in\sea{t}$ are such that $q\neq r$ then $\dpt{q}\neq\dpt{r}$. We now prove this by considering the converse: that $\dpt{q}=\dpt{r}$. By definition of $\sea{t}$ we have $\mmc{q}{t}\leq\con^{\dpt{q}}/(\awn\nna)$ and $\mmc{r}{t}\leq\con^{\dpt{r}}/(\awn\nna)$ so by the triangle inequality we have:
\be
\mmc{q}{r}\leq\mmc{q}{t}+\mmc{r}{t}\leq\con^{\dpt{q}}/(\awn\nna)+\con^{\dpt{r}}/(\awn\nna)=\con^{\dpt{q}}/\nna
\ee
which, since $\dpt{q}=\dpt{r}$\,,  contradicts Lemma \ref{lem1}.
\end{proof}

\begin{lemma}\label{fixlem4}
For all $t\in\ctb$ we have $\con^{\dpt{t}}\geq\ndo{t}\con/(1+\zct)$
\end{lemma}

\begin{proof}
By definition of $\ctb$ we immediately have that either $t\notin\cts$ or $\pnt{t}\notin\cts$. By Lemma \ref{cnanclem2} we then have that $\pnt{t}\notin\cts$. Hence, by definition of $\cts$ and since $\dpt{\pnt{t}}=\dpt{t}-1$\,, we can choose $r,s\in[\ntr]\setminus\mrg$ with $\poc{r}\neq\poc{s}$ and $\mmc{r}{\pnt{t}}\leq\zct\con^{\dpt{t}-1}$ and $\mmc{s}{\pnt{t}}\leq\zct\con^{\dpt{t}-1}$. Since $\poc{r}\neq\poc{s}$ we can, without loss of generality, assume that $\poc{s}\neq\poc{t}$ which means, since $s\notin\mrg$, that $\mmc{s}{t}\geq\ndo{t}$. By the triangle inequality and the fact that $\mmc{t}{\pnt{t}}\leq\con^{\dpt{t}-1}$ we then have:
\be
\ndo{t}\leq\mmc{s}{t}\leq\mmc{s}{\pnt{t}}+\mmc{\pnt{t}}{t}\leq\zct\con^{\dpt{t}-1}+\con^{\dpt{t}-1}=(1+\zct)\con^{\dpt{t}}/\con
\ee
Rearranging then gives us the desired result.
\end{proof}

\begin{lemma}\label{fixlem5}
For all $s,t\in\ctn$ with $s\neq t$ we have $\mmc{s}{t}>\zco\min(\ndo{s},\ndo{t})$.
\end{lemma}

\begin{proof}
Without loss of generality assume $\dpt{s}\geq\dpt{t}$. If $\dpt{s}=\dpt{t}$ then we have, from Lemma \ref{lem1}, that $\mmc{s}{t}>\con^{\dpt{t}}/\nna$. On the other hand, if $\dpt{s}>\dpt{t}$ then we have, from definition of $\ctn$ and the fact that $s\in\ctb$, that: 
\be
\mmc{s}{t}>(\con^{\dpt{t}}-\con^{\dpt{s}})/(\awn\nna)\geq(\con^{\dpt{t}}-\con^{\dpt{t}+1})/(\awn\nna)=(1-\con)\con^{\dpt{t}}/(\awn\nna)
\ee In either case we have that $\mmc{s}{t}>(1-\con)\con^{\dpt{t}}/(\awn\nna)$. From Lemma \ref{fixlem4} we then have that:
\be
\mmc{s}{t}>(1-\con)\con^{\dpt{t}}/(\awn\nna)\geq\ndo{t}(1-\con)\con/(\awn\nna(1+\zct))=\zco\ndo{t}\leq\zco\min(\ndo{s},\ndo{t})
\ee
\end{proof}

\begin{lemma}\label{fixlem6}
We have $|\ctn|\leq\nco$
\end{lemma}

\begin{proof}
Immediate from Lemma \ref{fixlem5} and the definition of $\nco$.
\end{proof}

\begin{lemma}\label{fixlem7}
For all $t\in\ctb$ we have $\dpt{t}\in\bo(\ln(1/\nmg))$
\end{lemma}

\begin{proof}
By Lemma \ref{fixlem4} and definition of $\nmg$ we have: 
\be
\con^{\dpt{t}}\geq\ndo{t}\con/(1+\zct)\geq\nmg\con/(1+\zct)
\ee 
Taking logarithms gives us the result.
\end{proof}

\begin{lemma}\label{fixlem8}
We have $|\ctb|\in\bo(\nco\ln(1/\nmg))$
\end{lemma}

\begin{proof}
By Lemma \ref{fixlem2} we have:
\be
\ctb=\ctn\cup\bigcup_{t\in\ctn}\sea{t}
\ee
so that:
\be
|\ctb|\leq|\ctn|+\sum_{t\in\ctn}|\sea{t}|
\ee
By lemmas \ref{fixlem3} and \ref{fixlem7} we have $|\sea{t}|\in\bo(\ln(1/\nmg))$ for all $t\in\ctn$. Substituting into the above inequality gives us $|\ctb|\leq\bo(|\ctn|\ln(1/\nmg))$.
Lemma \ref{fixlem6} then gives us the result.
\end{proof}

\begin{lemma}\label{qnoinctblem}
Suppose we have some $t\in[\ntr]$ such that for all $s\in\anc{t}$ we have $s\notin\ctb$. Then $t\notin\cts$.
\end{lemma}

\begin{proof}
We prove by induction on $\dpt{t}$. If $\dpt{t}=0$ then we have $t=1$ so that we immediately have $t\notin\cts$ (since there exists $r,s\in[\ntr]\setminus\mrg$ with $\poc{r}\neq\poc{s}$). Given some $\lep\in\nat$ suppose that the inductive hypothesis holds for all $t$ with $\dpt{t}=\lep$. Now consider any $t$ with $\dpt{t}=\lep+1$. Note that for all $s\in\anc{\pnt{t}}$ we have $s\in\anc{t}$ so that $s\notin\ctb$. Since $\dpt{\pnt{t}}=\dpt{t}-1=\lep$ we then have, by the inductive hypothesis, that $\pnt{t}\notin\cts$. If it was the case that $t\in\cts$ we would then have, by definition of $\ctb$, that $t\in\ctb$. But since $t\in\anc{t}$ this would be a contradiction. Hence, $t\notin\cts$. This completes the inductive proof.
\end{proof}

\begin{lemma}\label{cnanclem1}
For all $t\in[\ntr]$ there exists an $s\in\ctb$ such that $t\in\des{s}\cup\anc{s}$.
\end{lemma}

\begin{proof}
Assume, for contradiction, the converse: that there exists no $s\in\ctb$ with $t\in\des{s}\cup\anc{s}$. This means that for all $s\in\des{t}\cup\anc{t}$ we have $s\notin\ctb$. So choose some $r\in\des{t}\cap\,\lef$. Since $\anc{r}\subseteq\des{t}\cup\anc{t}$ we have, for all $s\in\anc{r}$\,, that $s\notin\ctb$. By Lemma \ref{qnoinctblem} we hence have that $r\notin\cts$. But since $r\in\lef$ this would mean that $r\in\ctb$ which, since $t\in\des{r}\cup\anc{r}$\,, is a contradiction.
\end{proof}

\begin{definition}
Define the policy $\hpv\in[\nma]^\ntr$ inductively from $t=1$ to $t=\ntr$ such that:
\begin{itemize}
\item If $t\notin\cts$ and there does not exist some $s\in\cts$ with $t=\pnt{s}$ then $\hpc{t}:=\hpc{\pnt{t}}$ (or is arbitrary when $t=1$).
\item If $t\notin\cts$ and there exists some $s\in\cts$ with $t=\pnt{s}$ then choose $\hpc{t}$ such that there exists $r\in[\ntr]\setminus\mrg$ with $\mmc{r}{t}\leq\ball{\dpt{t}}$ and $\hpc{t}=\poc{r}$. Note that by definition of $\cts$ such an $r$ does indeed exist (but $\hpc{t}$ is not unique - we choose any valid $\hpc{t}$).
\item If $t\in\cts$ and there does not exist $s\in[\ntr]\setminus\mrg$ with $\mmc{s}{t}\leq\zct\con^{\dpt{t}}$, we have $\hpc{t}:=\hpc{\pnt{t}}$. Since $1\notin\cts$ this is defined.
\item If $t\in\cts$ and there exists $s\in[\ntr]\setminus\mrg$ with $\mmc{s}{t}\leq\zct\con^{\dpt{t}}$ then $\hpc{t}=\poc{s}$. Note that by definition of $\cts$ we have that $\hpc{t}$ is uniquely defined.
\end{itemize}
\end{definition}

\begin{lemma}\label{cnanclem3}
Given $t\in[\ntr]$ with $\hpc{t}\neq\hpc{\pnt{t}}$ there exists $s\in\ctb$ with $t\in\anc{s}$.
\end{lemma}

\begin{proof}
By Lemma \ref{cnanclem1} choose $s\in\ctb$ such that $t\in\dec{s}\cup\anc{s}$. Assume, for contradiction, that $t\notin\anc{s}$. Then we must have $t\in\dec{s}\setminus\{s\}$. This means that  $s\notin\lef$ and hence, by definition of $\ctb$, we have that $s\in\cts$. So since we have both $t\in\dec{s}$ and $\pnt{t}\in\dec{s}$ we have, by Lemma \ref{cnanclem2}, that both $t\in\cts$ and $\pnt{t}\in\cts$. Since $\hpc{t}\neq\hpc{\pnt{t}}$ we must then have, by definition of $\hpv$, that there exists $q\in[\ntr]\setminus\mrg$ with $\mmc{q}{t}\leq\zct\con^{\dpt{t}}$. Since $\mmc{t}{\pnt{t}}\leq\con^{\dpt{t}-1}$ we have, by the triangle inequality, that:
\be
\mmc{q}{\pnt{t}}\leq\mmc{q}{t}+\mmc{t}{\pnt{t}}\leq\zct\con^{\dpt{t}}+\con^{\dpt{t}-1}=(\zct\con+1)\con^{\dpt{t}-1}\leq\zct\con^{\dpt{t}-1}
\ee
so since $\dpt{\pnt{t}}=\dpt{t}-1$ and $q\notin\mrg$ we have, by definition of $\hpv$ and since $\pnt{t}\in\cts$, that $\hpc{\pnt{t}}=\poc{q}$. We also have,  by definition of $\hpv$ and since both $\mmc{q}{t}\leq\zct\con^{\dpt{t}}$ and $t\in\cts$, that $\hpc{t}=\poc{q}$. But this means that $\hpc{t}=\hpc{\pnt{t}}$ which is a contradiction. We have hence shown that $t\in\anc{s}$.
\end{proof}

\begin{lemma}\label{cnanclem4}
Given $t\in[\ntr]$ with $\hpc{t}\neq\hpc{\pnt{t}}$ we have that $t\in\ctb$ or that there exists $q\in\ctb$ such that $\pnt{q}=t$.
\end{lemma}

\begin{proof}
Assume, for contradiction, that $t\notin\ctb$ and there does not exist $q\in\ctb$ such that $\pnt{q}=t$. By Lemma \ref{cnanclem3} we have that there exists $s\in\ctb$ with $t\in\anc{s}$. Since $t\notin\ctb$ we have $t\neq s$ so since $t\in\anc{s}$ we have $t\in\anc{\pnt{s}}$. If $s\notin\lef$ then, by definition of $\ctb$, we have $\pnt{s}\notin\cts$ and if $s\in\lef$ we have, again by definition of $\ctb$, that $s\notin\cts$. In either case there exists $r\in\dec{t}$ with $r\notin\cts$. By Lemma \ref{cnanclem2} this means that $t\notin\cts$. So since $\hpc{t}\neq\hpc{\pnt{t}}$ we have, from definition of $\hpv$\,, that there exists some $q\in\cts$ with $\pnt{q}=t$. But since $t\notin\cts$ this would imply that $q\in\ctb$ which is a contradiction.
\end{proof}

\begin{definition}
Let $\van$ be the set of all $t\in[\ntr]$ such that $\hpc{t}\neq\poc{t}$ and $t\notin\cts$.
\end{definition}

\begin{lemma}\label{hbivlem1}
For all $t\in\van$ there exists $s\in\ctb$ such that $t\in\anc{s}$.
\end{lemma}

\begin{proof}
By Lemma \ref{cnanclem1} choose $s\in\ctb$ with $t\in\dec{s}\cup\anc{s}$. Assume, for contradiction, that $t\notin\anc{s}$. Then $t\in\dec{s}\setminus\{s\}$ so $s\notin\lef$. By definition of $\ctb$ this means that $s\in\cts$. By Lemma \ref{cnanclem2} we then have $t\in\cts$ which is a contradiction.
\end{proof}

\begin{lemma}\label{hbivlem2}
We have $|\van|\in\mathcal{O}(\nco\ln(1/\nmg)^2)$
\end{lemma}

\begin{proof}
Given $t\in\ctb$ we have, by Lemma \ref{fixlem7}, that $\dpt{t}\in\mathcal{O}(\ln(1/\nmg))$ and hence that $|\anc{t}|\in\mathcal{O}(\ln(1/\nmg))$. By lemmas \ref{hbivlem1} and \ref{fixlem8} we then have that:
\be
|\van|\leq\left|\bigcup_{t\in\ctb}\anc{t}\right|\leq\sum_{t\in\ctb}|\anc{t}|\in\mathcal{O}(|\ctb|\ln(1/\nmg))\in\mathcal{O}(\nco\ln(1/\nmg)^2)
\ee
as required.
\end{proof}

\begin{lemma}\label{mmcleqphieq1}
For all $s,t\in[\ntr]$ with $s\in\des{t}$ we have $\mmc{s}{t}\leq\str\zct\con^{\dpt{t}}$
\end{lemma}

\begin{proof}
We hold $s$ fixed and prove by reverse induction on $\dpt{t}$ (i.e. from $\dpt{s}$ to $0$). When $\dpt{t}=\dpt{s}$ we have $s=t$ and hence $\mmc{s}{t}=0$ so the result holds trivially. Now suppose, for some $\lep\in[\dpt{s}]$\,, it holds when $\dpt{t}=\lep$. We now show that it holds when $\dpt{t}=\lep-1$ which will complete the inductive proof. So take $t$ with $\dpt{t}=\lep-1$. Let $r$ be such that $s\in\des{r}$ and $\pnt{r}=t$. Note that we have $\dpt{r}=\lep$ so by the inductive hypothesis we have $\mmc{s}{r}\leq\str\zct\con^{\lep}$. Since $\mmc{r}{\pnt{r}}\leq\con^{\dpt{r}-1}$ we then have, by the triangle inequality, that:
\be
\mmc{s}{t}\leq\mmc{s}{r}+\mmc{r}{t}=\mmc{s}{r}+\mmc{r}{\pnt{r}}\leq\str\zct\con^{\lep}+\con^{\lep-1}=(\str\zct\con+1)\con^{\lep-1}=\str\zct\con^{\lep-1}
\ee
\end{proof}

\begin{lemma}\label{mmcleqphieq2}
For all $t\in\mrg\cap\cts$ we have $\hpc{t}\in\rad{t}$.
\end{lemma}

\begin{proof}
Let $\ars$ be the set of all $r\in[\ntr]$ such that there exists $q\in[\ntr]\setminus\mrg$ with $\mmc{q}{r}\leq\zct\con^{\dpt{r}}$. Define:
\be
s:=\operatorname{argmin}_{r\in\anc{t}\,\cap\,\cts}\dpt{r}~~~~~;~~~~~
v:=\operatorname{argmax}_{r\in\anc{t}\,\cap\,\ars}\dpt{r}
\ee
noting that these both exist since $t\in\anc{t}\,\cap\,\cts$ and $1\in\anc{t}\,\cap\,\ars$.
Since $1\notin\cts$ (which comes directly from the fact that there exists $q,r\in[\ntr]\setminus\mrg$ with $\poc{q}\neq\poc{r}$) we have that $s\neq1$ and hence $\pnt{s}$ exists so let $\lep:=\dpt{\pnt{s}}$. Since $\pnt{s}\in\anc{t}$ with $\dpt{\pnt{s}}<\dpt{s}$\,, we have $\pnt{s}\notin\cts$ so by definition of $\cts$ there exists $r\in[\ntr]\setminus\mrg$ with $\mmc{r}{\pnt{s}}\leq\zct\con^{\lep}$. This means that $\pnt{s}\in\ars$ and hence that $v\in\des{\pnt{s}}$. Define:
\be
w:=\operatorname{argmin}_{q\in[\ntr]\setminus\mrg}\mmc{q}{t}
\ee
so that $\mmc{w}{t}=\thd{t}$.

We have two cases:
\begin{itemize}
\item First consider the case that $v=t$. In this case we have $t\in\ars$ so that there exists $q\in[\ntr]\setminus\mrg$ such that $\mmc{q}{t}\leq\zct\con^{\dpt{t}}$. By definition of $w$ we have that $\mmc{w}{t}\leq\mmc{q}{t}$ so $\mmc{w}{t}\leq\zct\con^{\dpt{t}}$  and hence, by definition of $\hpv$ and since $t\in\cts$, we have $\hpc{t}=\poc{w}$. Since $\mmc{w}{t}<\bra\thd{t}$ we have, by definition of $\rad{t}$\,, that $\poc{w}\in\rad{t}$\,. Hence, $\hpc{t}\in\rad{t}$ as required.
\item Next consider the case that $v\neq t$. Choose $u\in[\ntr]$ as follows:
\begin{itemize}
\item If $v=\pnt{s}$ then since $s\in\cts$ and $\pnt{s}\notin\cts$ we can, by definition of $\hpv$, choose $u\in[\ntr]\setminus\mrg$ such that $\mmc{u}{\pnt{s}}\leq\zct\con^{\lep}$ and $\hpc{\pnt{s}}=\poc{u}$. Since $v=\pnt{s}$ we have $\mmc{u}{v}\leq\zct\con^{\dpt{v}}$ and $\hpc{v}=\poc{u}$.
\item If $v\neq\pnt{s}$ then, since $v\in\des{\pnt{s}}$\,, we have $v\in\dec{s}$ so, since $s\in\cts$, we have, by Lemma \ref{cnanclem2} that $v\in\cts$. So $v\in\cts\cap\ars$ and so, by definition of $\hpv$ and $\ars$, choose $u\in[\ntr]\setminus\mrg$ such that $\mmc{u}{v}\leq\zct\con^{\dpt{v}}$ and $\hpc{v}=\poc{u}$.
\end{itemize}
In either case we have $u\notin\mrg$\,, $\mmc{u}{v}\leq\zct\con^{\dpt{v}}$ and $\hpc{v}=\poc{u}$. Since $v\in\anc{t}\setminus\{t\}$ let $q\in\anc{t}$ be such that $\pnt{q}=v$.
We have, by Lemma \ref{mmcleqphieq1} and the triangle inequality, that:
\be
\mmc{w}{q}\leq\mmc{w}{t}+\mmc{t}{q}\leq\thd{t}+\str\zct\con^{\dpt{q}}
\ee
Since $q\in\anc{t}$ and $\dpt{q}>\dpt{v}$ we must have, by definition of $v$, that $q\notin\ars$ so since $w\notin\mrg$ we also have:
\be
\mmc{w}{q}>\zct\con^{\dpt{q}}
\ee
Substituting this inequality into the previous and rearranging gives us:
\be
\thd{t}>\zct\con^{\dpt{q}}-\str\zct\con^{\dpt{q}}=\zct(1-\str)\con^{\dpt{q}}
\ee
so that:
\be
\con^{\dpt{q}}<\thd{t}/(\zct(1-\str))
\ee
Since $\dpt{v}=\dpt{q}-1$ and $\mmc{q}{v}\leq\con^{\dpt{v}}$ (as $v=\pnt{q}$) we then have, from the triangle inequality and Lemma \ref{mmcleqphieq1}, that:
\begin{align*}
\mmc{t}{u}\leq\mmc{t}{q}+\mmc{q}{v}+\mmc{v}{u}&\leq\str\zct\con^{\dpt{q}}+\con^{\dpt{v}}+\zct\con^{\dpt{v}}\\
&=(\str\zct+(1+\zct)/\con)\con^{\dpt{q}}\\
&<(\str\zct+(1+\zct)/\con)\thd{t}/(\zct(1-\str))\\
&=\bra\thd{t}
\end{align*}
So that $\poc{u}\in\rad{t}$. Since $\hpc{v}=\poc{u}$\,, all that is left to do now is to prove that $\hpc{t}=\hpc{v}$. To prove this we need only show that for all $r\in(\dec{v}\setminus\{v\})\cap\anc{t}$ we have $\hpc{r}=\hpc{\pnt{r}}$. To show this take any such $r\in(\dec{v}\setminus\{v\})\cap\anc{t}$. Since $v\in\dec{\pnt{s}}$ we must have $r\in\dec{s}$ and hence, by Lemma \ref{cnanclem2} and the fact that $s\in\cts$\,, we have $r\in\cts$. Since $\dpt{r}>\dpt{v}$ and $r\in\anc{t}$ we have, by definition of $v$, that $r\notin\ars$. So $r\in\cts\setminus\ars$ and hence, by definition of $\hpv$ and $\ars$, we have $\hpc{r}=\hpc{\pnt{r}}$ as required.
\end{itemize}
\end{proof}

\begin{lemma}\label{vcpmalem}
We have:
\be
\sum_{t\in[\ntr]}\expt{\lct{t}{\hpc{t}}}\leq\sum_{t\in[\ntr]}\xll{t}+\mathcal{O}(\nco\ln(1/\nmg)^2)
\ee
\end{lemma}

\begin{proof}
Note first that given $t\in[\ntr]\setminus\van$ with $\hpc{t}\neq\poc{t}$ we must have $t\in\cts$ so by definition of $\hpv$, we have $t\in\mrg$ and hence $t\in\mrg\cap\cts$. This means that:
\be
\sum_{t\in[\ntr]}\lct{t}{\hpc{t}}\leq\sum_{t\in[\ntr]\setminus(\mrg\cap\,\cts)}\lct{t}{\poc{t}}+\sum_{t\in\mrg\cap\,\cts}\lct{t}{\hpc{t}}+\sum_{t\in\van}\lct{t}{\hpc{t}}
\ee
Given $t\in\mrg\cap\cts$ we have, by Lemma \ref{mmcleqphieq2} and definition of $\xll{t}$, that $\expt{\lct{t}{\hpc{t}}}\leq\xll{t}$. Substituting into the above inequality (after taking expectations) gives us:
\be
\sum_{t\in[\ntr]}\expt{\lct{t}{\hpc{t}}}\leq\sum_{t\in[\ntr]\setminus(\mrg\cap\,\cts)}\expt{\lct{t}{\poc{t}}}+\sum_{t\in\mrg\cap\,\cts}\xll{t}+|\van|\leq\sum_{t\in[\ntr]}\xll{t}+|\van|
\ee
Applying Lemma \ref{hbivlem2} then gives us the result.
\end{proof}

\begin{lemma}\label{cbnnbolem}
We have:
\be
\sum_{t\in[\ntr]}\expt{\lct{t}{\act{t}}}\leq\sum_{t\in[\ntr]}\lct{t}{\hpc{t}}+\tilde{\mathcal{O}}\left(\left(\pra-\frac{\nco\ln(1/\nmg)}{\pra}\right)\sqrt{\nma\ntr}\right)
\ee
\end{lemma}

\begin{proof}
By Lemma \ref{cnanclem4} we have:
\be
\sum_{t\in[\ntr]}\indi{\hpc{t}\neq\hpc{\pnt{t}}}\leq 2|\ctb|
\ee
so by Lemma \ref{fixlem8} we have:
\be
\sum_{t\in[\ntr]}\indi{\hpc{t}\neq\hpc{\pnt{t}}}\in\mathcal{O}(\nco\ln(1/\nmg))
\ee
The regret bound of \cbn\ then implies the result.
\end{proof}

Combining lemmas \ref{vcpmalem} and \ref{cbnnbolem} gives us the loss bound in Theorem \ref{mainth}. The time complexity comes from the fact that:
\begin{itemize}
\item When $\mm$ has bounded doubling dimension and $\nna>1$ we have that the per-trial time complexity of adaptive $\nna$-nearest neighbour search (when using a navigating net) is in $\mathcal{O}(\ln(1/\asp))$.
\item For all $t\in[\ntr]$ we have that $\dpt{t}\in\mathcal{O}(\ln(1/\asp))$ and hence only $\mathcal{O}(\ln(1/\asp))$ approximate nearest neighbour searches need to be performed per trial.
\item The per-trial time complexity of \cbn\ is only $\mathcal{O}(\ln(\ntr)^2\ln(\nma))$.
\end{itemize}
\hfill $\blacksquare$

\section{Proof of Theorem \ref{maincor}}

\nc{\xip}{\bbr}
\nc{\wid}{J}
\nc{\itt}[1]{i_{#1}}
\nc{\jtt}[1]{j_{#1}}
\nc{\utt}[1]{q_{#1}}
\nc{\btt}[1]{b_{#1}}
\nc{\sss}[2]{\ars_{#1,#2}}
\nc{\trad}{r'}
\nc{\width}{w}

We will now analyse Theorem \ref{mainth} when choosing our margin $\mrg$ as given in the statement of Theorem \ref{maincor} and choosing our policy $\pov$ such that for all $t\in[\ntr]$ we have $\poc{t}:=\ext(\xt{t})$.

\begin{definition}
For all $t\in[\ntr]$ let $\utt{t}$ be the minimiser of $\mmc{s}{t}$ out of all $s\in[\ntr]\setminus\mrg$ with $\poc{s}\neq\poc{t}$. Since $\ext(\xt{t})\neq\ext(\xt{\utt{t}})$ choose $\btt{t}\in\dbn$ such that $\btt{t}$ lies on the straight line from $\xt{t}$ to $\xt{\utt{t}}$. Since $\btt{t}\in\dbn$ choose $\itt{t}\in[\nmb]$ such that $\btt{t}\in\bll{\zi{\itt{t}}}{\ri{\itt{t}}}$.
\end{definition}

\begin{definition}
Define $\wid:=\cnst\log_2(\ntr)$. For all $t\in[\ntr]$ define $\jtt{t}$ as the minimum number in $[\wid]\cup\{0\}$ such that $\xt{t}\in\bll{\zi{\itt{t}}}{2^{\jtt{t}}\xip\ri{\itt{t}}}$. Note that since $\ri{\itt{t}}\geq\ntr^{-\cnst}$ this is defined.
\end{definition}

\begin{lemma}\label{corlem1}
Given $t\in[\ntr]$ with $\jtt{t}>0$ we have $\ndo{t}\geq2^{\jtt{t}-1}(\xip-1)\ri{\itt{t}}$.
\end{lemma}

\begin{proof}
By definition of $\jtt{t}$ we have that $\xt{t}\notin\bll{\zi{\itt{t}}}{2^{\jtt{t}-1}\xip\ri{\itt{t}}}$ so since $\btt{t}\in\bll{\zi{\itt{t}}}{\ri{\itt{t}}}$ we have, by the triangle inequality, that:
\be
\eun{\xt{t}-\btt{t}}\geq\eun{\xt{t}-\zi{\itt{t}}}-\eun{\btt{t}-\zi{\itt{t}}}\geq2^{\jtt{t}-1}\xip\ri{\itt{t}}-\ri{\itt{t}}\geq2^{\jtt{t}-1}(\xip-1)\ri{\itt{t}}
\ee
Since $\btt{t}$ is on the straight line from $\xt{t}$ to $\xt{\utt{t}}$ we have $\eun{\xt{t}-\xt{\utt{t}}}\geq\eun{\xt{t}-\btt{t}}$. By definition of $\utt{t}$ we have $\eun{\xt{t}-\xt{\utt{t}}}=\mmc{t}{\utt{t}}=\ndo{t}$. Putting together gives us:
\be
\ndo{t}=\eun{\xt{t}-\xt{\utt{t}}}\geq\eun{\xt{t}-\btt{t}}\geq2^{\jtt{t}-1}(\xip-1)\ri{\itt{t}}
\ee
as required.
\end{proof}

\begin{lemma}\label{corlem2}
For all $t\in[\ntr]$ with $\jtt{t}=0$ we have $\ndo{t}>2^{\jtt{t}-1}(\bbr-1)\ri{\itt{t}}$
\end{lemma}

\begin{proof}
Since $\utt{t}\notin\mrg$ we have $\utt{t}\notin\bll{\zi{\itt{t}}}{\bbr\ri{\itt{t}}}$ so since $\btt{t}\in\bll{\zi{\itt{t}}}{\ri{\itt{t}}}$ we have, by the triangle inequality, that:
\be
\eun{\xt{\utt{t}}-\btt{t}}\geq\eun{\xt{\utt{t}}-\zi{\itt{t}}}-\eun{\btt{t}-\zi{\itt{t}}}\geq\bbr\ri{\itt{t}}-\ri{\itt{t}}=(\bbr-1)\ri{\itt{t}}
\ee
Since $\btt{t}$ is on the straight line from $\xt{t}$ to $\xt{\utt{t}}$ we have $\eun{\xt{t}-\xt{\utt{t}}}\geq\eun{\xt{\utt{t}}-\btt{t}}$. By definition of $\utt{t}$ we have $\eun{\xt{t}-\xt{\utt{t}}}=\mmc{t}{\utt{t}}=\ndo{t}$. Putting together gives us:
\be
\ndo{t}=\eun{\xt{t}-\xt{\utt{t}}}\geq\eun{\xt{\utt{t}}-\btt{t}}\geq(\bbr-1)\ri{\itt{t}}>2^{\jtt{t}-1}(\bbr-1)\ri{\itt{t}}
\ee
as required.
\end{proof}

\begin{definition}
Let $\ars$ be a subset of $[\ntr]$ of maximum cardinality subject to the condition that for all $s,t\in\ars$ with $s\neq t$ we have $\mmc{s}{t}>\zco\min(\ndo{s},\ndo{t})$.
\end{definition}

\begin{definition}
For all $i\in[\nmb]$ and $j\in[\wid]\cup\{0\}$ define:
\be
\sss{i}{j}=\set{t\in\ars}{\itt{t}=i\,\wedge\,\jtt{t}=j}
\ee
\end{definition}

\begin{lemma}\label{corlem3}
For all $i\in[\nmb]$ and $j\in[\wid]\cup\{0\}$ we have $|\sss{i}{j}|\in\mathcal{O}(1)$
\end{lemma}

\begin{proof}
Let $\trad:=2^{j}\xip\ri{i}$ and $\width:=\zco(\xip-1)/2\xip$. By lemmas \ref{corlem1} and \ref{corlem2} we have, for all $t\in\sss{i}{j}$, that:
\be
\ndo{t}\geq2^{\jtt{t}-1}(\xip-1)\ri{\itt{t}}=2^{j-1}(\xip-1)\ri{i}=\width\trad/\zco
\ee
so, by definition of $\ars$, we have, for all $s,t\in\sss{i}{j}$ with $s\neq t$, that $\eun{\xt{s}-\xt{t}}>\width\trad$. Also, for all $t\in\sss{i}{j}$ we have, by definition of $\jtt{t}$, that:
\be
\xt{t}\in\bll{\zi{\itt{t}}}{2^{\jtt{t}}\xip\ri{\itt{t}}}=\bll{\zi{i}}{2^{j}\xip\ri{i}}=\bll{\zi{i}}{\trad}
\ee
So all the elements of $\sss{i}{j}$ are contained in a ball of radius $\trad$ and are all of distance at least $\width\trad$ apart. Since $\width$ is a positive constant and the dimensionality is a constant we have the result.
\end{proof}

\begin{lemma}\label{corlem4}
We have $\nco\in\mathcal{O}(\nmb\ln(\ntr))$.
\end{lemma}

\begin{proof}
We have:
\be
\ars=\bigcup_{i\in[\nmb]}\bigcup_{j\in[\wid]\cup\{0\}}\sss{i}{j}
\ee
so that by Lemma \ref{corlem3} we have $|\ars|\in\mathcal{O}(\nmb\wid)$. Since $\cnst$ is a constant we have $\wid\in\mathcal{O}(\ln(\ntr))$ and hence $|\ars|\in\mathcal{O}(\nmb\ln(\ntr))$. By definition of $\nco$ and $\ars$ we have that $\nco=|\ars|$ which completes the proof.
\end{proof}

\begin{lemma}\label{corlem5}
We have:
\be
\sum_{t\in[\ntr]}\xll{t}\leq\sum_{t\in[\ntr]\setminus\mrg}\expt{\lct{t}{\ext(\xt{t})}}+|\mrg|
\ee
\end{lemma}

\begin{proof}
For all $t\in[\ntr]\setminus\mrg$ we have, by definition of $\xll{t}$, that $\xll{t}=\expt{\lct{t}{\poc{t}}}=\expt{\lct{t}{\ext(\xt{t})}}$. For all $t\in\mrg$ we immediately have that $\xll{t}\leq 1$. So for all $t\in[\ntr]$ we have:
\be
\xll{t}\leq\indi{t\notin\mrg}\expt{\lct{t}{\ext(\xt{t})}}+\indi{t\in\mrg}
\ee
Summing over all $t\in[\ntr]$ gives us the result.
\end{proof}

\begin{lemma}\label{corlem6}
We have $\ln(1/\nmg)\in\mathcal{O}(\ln(\ntr))$
\end{lemma}

\begin{proof}
Let $t$ be the element of $[\ntr]$ that minimises $\ndo{t}$ so that $\nmg=\ndo{t}$. By lemmas \ref{corlem1} and \ref{corlem2} we have
\be
\nmg=\ndo{t}\geq2^{\jtt{t}-1}(\xip-1)\ri{\itt{t}}\geq\ri{\itt{t}}(\xip-1)/2\in\mathcal{O}(\ri{\itt{t}})
\ee
so since $\ri{\itt{t}}\geq\ntr^{-\cnst}$ (where $\cnst$ is a constant) we have the result.
\end{proof}

Since the desired bound is vacuous when $\nmb>\ntr$ we can assume otherwise so that by lemmas \ref{corlem4} and \ref{corlem6} we have: 
\be
\nco\ln(1/\nmg)^2\in\mathcal{O}(\nmb\ln(\ntr)^3)\subseteq\tilde{\mathcal{O}}(\sqrt{\nmb\ntr})\subseteq\tilde{\mathcal{O}}((\pra-\nmb/\pra)\sqrt{\ntr})
\ee  Substituting this and lemmas \ref{corlem4}, \ref{corlem5} and \ref{corlem6} into Theorem \ref{mainth} then gives us the result.
\hfill $\blacksquare$

\bibliographystyle{abbrv}
\bibliography{NNbib}

\end{document}